%% file: main.tex
\documentclass{article}
\usepackage{amsmath}
\usepackage{amsfonts}
\DeclareMathOperator*{\argmin}{arg\,min}

\usepackage{microtype}
\usepackage{graphics}
\usepackage{graphicx}
\usepackage{subcaption}
\usepackage{array}
\usepackage{booktabs} 
\usepackage{hyperref}
\usepackage{amsfonts}
\usepackage{caption}
\usepackage{wrapfig}
\usepackage{amsthm}
\usepackage{thm-restate}

\usepackage{enumitem}

\declaretheorem[name=Definition]{definition}
\declaretheorem[name=Lemma]{lemma}
\declaretheorem[name=Theorem]{theorem}

\usepackage[accepted]{icml2023}

\icmltitlerunning{Interpreting Robustness Proofs of Deep Neural Networks}

\begin{document}

\twocolumn[
\icmltitle{Interpreting Robustness Proofs of Deep Neural Networks}



\icmlsetsymbol{equal}{*}

\begin{icmlauthorlist}
\icmlauthor{Debangshu~Banerjee}{to}
\icmlauthor{Avaljot~Singh}{to}
\icmlauthor{Gagandeep~Singh}{to,goo}
\end{icmlauthorlist}

\icmlaffiliation{to}{University of Illinois Urbana-Champaign}
\icmlaffiliation{goo}{VMware Research}

\icmlcorrespondingauthor{Debangshu~Banerjee}{db21@illinois.edu}

\icmlkeywords{Machine Learning, ICML}

\vskip 0.3in
]



\printAffiliationsAndNotice{} 
\input{defs}

\begin{abstract}
In recent years numerous methods have been developed to formally verify the robustness of deep neural networks (DNNs). 
Though the proposed techniques are effective in providing mathematical guarantees about the DNNs behavior, it is not clear whether the proofs generated by these methods are human-interpretable. 
In this paper, we bridge this gap by developing new concepts, algorithms, and representations to generate human understandable interpretations of the proofs. 
Leveraging the proposed method, we show that the robustness proofs of standard DNNs rely on spurious input features, while the proofs of DNNs trained to be provably robust filter out even the semantically meaningful features. The proofs for the DNNs combining adversarial and provably robust training are the most effective at selectively filtering out spurious features as well as relying on human-understandable input features. 
\end{abstract}

\input{introduction}
\input{related.tex}
\input{prelimeries}
\input{sparseFeatureComputation}
\input{experimentalEval}
\input{conclusion.tex}

\clearpage

\bibliography{main}
\bibliographystyle{icml2023}
\appendix
\onecolumn
\input{proofs.tex}
\input{appendixb.tex}

\end{document}

%% file: defs.tex
\definecolor{WowColor}{rgb}{.75,0,.75}
\definecolor{SubtleColor}{rgb}{0,0,.50}

\newcommand{\aval}[1]{{\textcolor{blue}{[Aval: #1]}}}
\newcommand{\algname}{{SuPFEx }}
\newcommand{\lb}{\Lambda}
\newcommand{\inreg}{\phi}
\newcommand{\outprop}{\psi}
\newcommand*{\Relu}{\mathit{ReLU}}
\newcommand*{\ver}{\mathcal{V}}
\newcommand*{\abstraction}{\mathcal{A}}
\newcommand*{\absFeature}{\mathcal{F}}
\newcommand*{\approxnet}{N^{a}}
\newcommand*{\suffFeature}{\absFeature_{S}}
\newcommand*{\contri}{\Delta}
\newcommand*{\fset}[1]{\mathbb{S}(#1)}
\newcommand*{\grad}{\mathcal{G}}
\newcommand*{\alg}{\mathbb{A}}
\newcommand*{\dlower}{\delta}
\newcommand*{\expect}{\mathop{\mathbb{E}}}
\newcommand*{\cannFeature}{\absFeature_{S_{0}}}
\newcolumntype{C}[1]{>{\centering\arraybackslash}c{#1}}

%% file: introduction.tex
\section{Introduction}
The black box construction and lack of robustness of deep neural networks (DNNs) are major obstacles to their real-world deployment in safety-critical applications like autonomous driving~\cite{bojarski2016end_DUP} or medical diagnosis~\cite{AMATO201347}.
To mitigate the lack of trust caused by black-box behaviors, there has been a large amount of work on interpreting individual DNN predictions to gain insights into their internal workings. Orthogonally, the field of DNN verification has emerged to formally prove or disprove the robustness of neural networks in a particular region capturing an infinite set of inputs. Verification can be leveraged during training for constructing more robust models. 

We argue that while these methods do improve trust to a certain degree, the insights and guarantees derived from their independent applications are not enough to build sufficient confidence for enabling the reliable real-world deployment of DNNs. 
Existing DNN interpretation methods~\cite{sundararajan2017axiomatic} explain the model behavior on individual inputs, but they often do not provide human-understandable insights into the workings of the model on an infinite set of inputs handled by verifiers. 
Similarly, the DNN verifiers~\cite{singh2018robustness, zhang2018crown} can generate formal proofs capturing complex invariants sufficient to prove network robustness but it is not clear whether these proofs are based on any meaningful input features learned by the DNN that are necessary for correct classification. This is in contrast to standard program verification tasks where proofs capture the semantics of the program and property. In this work, to improve trust, we propose for the first time, the problem of interpreting the invariants captured by DNN robustness proofs. 


\noindent \textbf{Key Challenges.}
The proofs generated by state-of-the-art DNN verifiers 
encode high-dimensional complex convex shapes defined over thousands of neurons in the DNN. It is not exactly clear how to map these shapes to human understandable interpretations. 
Further, certain parts of the proof may be more important for it to hold than the rest. 
Thus we need to define a notion of importance for different parts of the proof and develop methods to identify them. 


\noindent \textbf{Our Contributions. }
We make the following contributions to overcome these challenges and develop a new method for interpreting DNN robustness proofs.
\begin{itemize}[noitemsep,nolistsep,leftmargin=*]
    \vspace{-0.2cm}
    \item We introduce a novel concept of proof features that can be analyzed independently by generating the corresponding interpretations. A priority function is then associated with the proof features
    that signifies their importance 
    in the complete proof. 
    \item We design a general algorithm called \textbf{\textit{\algname}} (\textbf{Su}fficient \textbf{P}roof \textbf{F}eature \textbf{Ex}traction) that extracts a set of proof features that retain only the more important parts of the proof while still proving the property. 
    \item We compare interpretations of the proof features for standard DNNs and state-of-the-art robustly trained DNNs for the MNIST and CIFAR10 datasets. We observe that the proof features corresponding to the standard networks rely on spurious input features while the proofs of adversarially trained DNNs~\cite{madry:17} filter out some of the spurious features. In contrast, the networks trained with certifiable training~\cite{ZhangCXGSLBH20} produce proofs that do not rely on any spurious features but they also miss out on some meaningful features. Proofs for training methods that combine both empirical robustness and certified robustness~\cite{BalunovicV:20} provide a common ground. They not only rely on human interpretable features but also selectively filter out the spurious ones. We also empirically show that these observations are not contingent on any specific DNN verifier. 
\end{itemize}

%% file: related.tex
\section{Related Work}
We discuss prior works related to ours.
\\
\textbf{DNN interpretability.} There has been an extensive effort to develop interpretability tools for investigating the internal workings of DNNs. 
These include feature attribution techniques like saliency maps \cite{sundararajan2017axiomatic, DBLP:journals/corr/SmilkovTKVW17}, using surrogate models to interpret local decision boundaries \cite{ribeiro2016should}, finding influential \cite{koh2017understanding}, prototypical \cite{kim2016examples}, or counterfactual inputs \cite{goyal2019counterfactual}, training sparse decision layers \cite{wong2021leveraging}, utilizing robustness analysis \cite{hsieh2021evaluations}.
Most of these interpretability tools focus on generating local explanations that investigate how DNNs work on individual inputs.
Another line of work, rather than explaining individual inputs, tries to identify specific concepts associated with a particular neuron~\cite{DBLP:journals/corr/SimonyanVZ13, bau2020understanding}. 
However, to the best of our knowledge, there is no existing work that allows us to interpret DNN robustness proofs.
\\
\textbf{DNN verification.}
Unlike DNN interpretability methods, prior works in DNN verification focus on formally proving whether the given DNN satisfies desirable properties like robustness \cite{singh2018robustness, wang2021betacrown}, fairness \cite{DBLP:conf/sas/MazzucatoU21}, etc. 
The DNN verifiers are broadly categorized into three main categories - (i) sound but incomplete verifiers which may not always prove property even if it holds~\cite{gehr2018ai2,singh2018fast,singh2019abstract,singh2019beyond,zhang2018crown,Lirpa:20,DBLP:conf/nips/SalmanY0HZ19}, (ii) complete verifiers that can always prove the property if it holds ~\cite{wang2018neurify, gehr2018ai2, bunel2020branch, bunel2020efficient, DBLP:conf/cav/BakTHJ20, ehlers2017formal, ferrari2022complete, fromherz2021fast, wang2021beta, depalma2021scaling, anderson2020strong, zhang2022general} and (iii) verifiers with probabilistic guarantees \cite{randSmooth}. 
\\
\textbf{Robustness and interpretability.}
Existing works \cite{madry:17, BalunovicV:20, ZhangCXGSLBH20} in developing robustness training  methods for neural networks provide a framework to produce networks that are inherently immune to adversarial perturbations in input. 
Recent works \cite{tsipras2018robustness, pmlr-v97-zhang19p} also show that there may be a robustness-accuracy tradeoff that prevents highly robust models achieve high accuracy.
Further, in \cite{tsipras2018robustness} authors show that networks trained with adversarial training methods learn fundamentally different input feature representations than standard networks where the adversarially trained networks capture more human-aligned data characteristics.





%% file: prelimeries.tex
\section{Preliminaries}
\label{sec:prelims}
In this section, we provide the necessary background on DNN verification and existing works on traditional DNN interpretability with sparse decision layers.
While our method is applicable to general architectures, for simplicity, we focus on a $l$-layer feedforward DNN $N : \mathbb{R}^{d_{0}} \rightarrow \mathbb{R}^{d_{l}}$ for the rest of this paper. Each layer $i$ except the final one applies the transformation $X_i=\sigma_i(W_i\cdot X_{i-1}+B_i)$ where $W_i \in \mathbb{R}^{d_{i} \times d_{i-1}}$ and $B_i\in \mathbb{R}^{d_{i}}$ are respectively the weights and biases of the affine transformation and $\sigma_i$is a non-linear activation like ReLU, Sigmoid, etc. corresponding to layer $i$.   
The final layer only applies the affine transformation and the network output is a vector $Y=W_l \cdot X_{l-1} + B_l$. 
\\
\noindent \textbf{DNN verification.}
At a high level, DNN verification involves proving that the network outputs $Y=N(X)$ corresponding to all inputs $X$ from an input region specified by $\inreg$, satisfy a logical specification $\outprop$.
A common property is - the local robustness where the output specification $\outprop$ is defined as 
linear inequality over the elements of the output vector of the neural network.   
The output specification, in this case, is written as $\outprop(Y) = (C^{T} Y \geq 0)$ where $C \in \mathbb{R}^{d_{l}}$ defines the linear inequality for encoding the robustness property.
For the rest of the paper, we refer to the input region $\inreg$ and output specification $\outprop$ together as \textit{property}  $(\inreg, \outprop)$. 
\\
Next, we briefly discuss how DNN robustness verifiers work. 
A DNN verifier $\ver$ 
symbolically computes a possibly over-approximated output region $\abstraction \subseteq \mathbb{R}^{d_{l}}$ containing all possible outputs of $N$ corresponding to $\inreg$.
Let, $\lb(\abstraction) = \min_{Y \in \abstraction} C^TY$ denote the minimum value of $C^TY$ where $Y \in \abstraction$. Then $N$ satisfies property $(\inreg, \outprop)$ if $\lb(\abstraction) \geq 0$.
Most existing DNN verifiers~\cite{singh2018fast,singh2019abstract,zhang2018crown} are exact for affine transformations. 
However, for non-linear activation functions, these verifiers compute convex regions that over-approximate the output of the activation function.
Note that, due to the over-approximations, DNN verifiers are sound but not complete - the verifier may not always prove property even if it holds.
For piecewise linear activation functions like ReLU, complete verifiers exist handling the activation exactly, which in theory always prove a property if it holds.
Nevertheless, complete verification in the worst case takes exponential time, making them practically infeasible.
In the rest of the paper, we focus on deterministic, sound, and incomplete verifiers which are more scalable than complete verifiers. \\
\noindent \textbf{DNN interpretation with sparse decision layer.} 
DNNs considered in this paper, have complex multi-layer structures, making them harder to interpret. 
Instead of interpreting what each layer of the network is doing, recent works \cite{wong2021leveraging, liao2022automated} treat DNNs as the composition of a \textit{deep feature extractor} and an affine \textit{decision layer}.
The output of each neuron of the penultimate layer represents a single deep feature and the final affine layer linearly combines these deep features to compute the network output.
This perspective enables us to identify the set of features used by the network to compute its output and to investigate their semantic meaning using the existing feature visualization techniques \cite{ribeiro2016should, DBLP:journals/corr/SimonyanVZ13}.
However, visualizing each feature is practically infeasible for large DNNs where the penultimate layer can contain hundreds of neurons.
To address this, the work of \cite{wong2021leveraging} tries to identify a smaller set of features that are sufficient to maintain the perfomance of the network. 
This smaller but sufficient feature set retains only the most important features corresponding to a given input. 
It is shown empirically \cite{wong2021leveraging} 
that a subset of these features of size less than 10 is sufficient to maintain the accuracy of state-of-the-art models.


%% file: sparseFeatureComputation.tex
\section{Interpreting DNN Proofs}
Next, we describe our approach for interpreting DNN robustness proofs.

\noindent \textbf{Proof features.} Similar to traditional DNN interpretation described above, for proof interpretation, we propose to segregate the final decision layer from the network and look at the features extracted at the penultimate layer. 
However, DNN verifiers work on an input region ($\inreg$) consisting of infinitely many inputs instead of a single input as handled by existing work. 
In this case, for a given input region $\inreg$, we look at the symbolic shape (for example - intervals, zonotopes, polytopes, etc.) computed by the verifier at the penultimate layer and then compute its projection on each dimension of the penultimate layer. 
These projections yield an interval $[l_n, u_n]$ which contains all possible output values of the corresponding neuron $n$ with respect to $\phi$.  
\vspace{-0.5mm}
\begin{definition}[Proof Features]
Given a network $N$, input region $\inreg$ and neural network verifier $\ver$, for each neuron $n_{i}$ at the  penultimate layer of $N$, the proof feature $\absFeature_{n_{i}}$ extracted at that neuron $n_i$ is an interval $[l_{n_{i}}, u_{n_{i}}]$ such that $\forall X \in \inreg$, the output of $n_{i}$ always lies in the range $[l_{n_{i}}, u_{n_{i}}]$. 
\end{definition}
Note that, the computation of the proof features is verifier dependent, i.e., for the same network and input region, different verifiers may compute different values $l_n$ and $u_n$ for a particular neuron $n$.
For any input region $\inreg$, the first $(l - 1)$ layers of $N$ along with the verifier $\ver$ act as the \textbf{proof feature extractor}.
For the rest of this paper, we use $\absFeature$ to denote the set of all proof features at the penultimate layer and $\suffFeature$ to denote the proof features corresponding to $S \subseteq [d_{l-1}]$.  
\begin{align*}
    \suffFeature = \{ \absFeature_{n_{i}} \;|\; i \in S\}
\end{align*}
Suppose $N$ is formally verified by the verifier $\ver$ to satisfy the property ($\inreg$, $\outprop$). 
Then in order to gain insights about the proof generated by $\ver$, we can directly investigate (described in section \ref{sec:featuresInvestigation}) the extracted proof features $\absFeature$. 
However, the number of proof features for contemporary networks can be very large (in hundreds). Many of these features may be spurious and not important for the proof. Similar to how network interpretations are generated when classifying individual inputs, we want to identify a smaller set of proof features that are more important for the proof of the property ($\inreg$, $\outprop$).
The key challenge here is defining the most important set of proof features w.r.t the property $(\inreg, \outprop)$.

\subsection{Sufficient Proof Features}
We argue that a \textit{minimum} set of proof features $\cannFeature \subseteq \absFeature$ that can prove the property $(\inreg, \outprop)$ with verifier $\ver$ contains an important set of proof features w.r.t $(\inreg, \outprop)$. 
The minimality of $\cannFeature$ enforces that it can only retain the proof features that are essential to prove the property.
Otherwise, it would be possible to construct a smaller set of proof features that preserves the property violating the minimality of $\cannFeature$.
Leveraging this hypothesis, we can model extracting a set of important proof features as computing a minimum proof feature set capable of preserving the property $(\inreg, \outprop)$ with $\ver$.
To identify a minimum proof feature set, we introduce the novel concepts of proof feature pruning and sufficient proof features below: 
\begin{definition}[Proof feature Pruning]
\label{def:pruning}
Pruning any Proof feature $\absFeature_{n_i} \in \absFeature$ corresponding to neuron $n_{i}$ in the penultimate layer involves setting weights of all its outgoing connections to 0 so that given any input $X \in \inreg$ the final output of $N$ no longer depends on the output of 
 ${n_i}$.
\end{definition}
Once, a proof feature $\absFeature_{n_i}$ is pruned the verifier $\ver$ no longer uses $\absFeature_{n_i}$ to prove the property $(\inreg, \outprop)$.
\begin{definition}[Sufficient proof features]
\label{def:sufficent}
For the proof of property $(\inreg, \outprop)$ on DNN $N$ with verifier $\ver$, a nonempty set $\suffFeature \subseteq \absFeature$ of proof features is sufficient if  the property still holds with verifier $\ver$ even if all the proof features \textbf{not in} $\suffFeature$ are pruned. 
\end{definition}
\begin{definition}[Minimum proof features]
Minimum proof feature set $\cannFeature \subseteq \absFeature$ for a network $N$ verified with $\ver$ on $(\inreg, \outprop)$ is a sufficient proof feature set containing the minimum number of proof features. 
\end{definition}
Extracting a minimum set of proof features $\cannFeature$ from $\absFeature$ is equivalent to pruning maximum number of proof features from $\absFeature$ without violating the property $(\inreg, \outprop)$. 
Let, $W_{l}[:,i] \in \mathbb{R}^{d_{l}}$ denote the $i$-th column of the weight matrix $W_l$ of the final layer $N_{l}$. Pruning any proof feature $\absFeature_{n_{i}}$ results in setting all weights in $W_l[:,i]$ to 0.
Therefore, to compute $\cannFeature$, it is sufficient to devise an algorithm that can prune maximum number of columns from $W_l$ while still preserving the property $(\inreg, \outprop)$.
\\
For any proof feature set $\absFeature_{S} \subseteq \absFeature$, let $W_{l}(S) \in \mathbb{R}^{d_{l} \times d_{l-1}}$ be the weight matrix of the pruned final layer that only retains proof features corresponding to $\suffFeature$. 
Then columns of $W_{l}(S)$ are defined as follows where $\underline{0} \in \mathbb{R}^{d_{l-1}}$ dentoes a constant all-zero vector
\begin{align}
\label{eq:prunedLayer}
W_{l}(S)[:,i] = \begin{cases} 
          W_{l}[:,i] & i \in S \\
          \;\;\underline{0} & \text{otherwise} 
       \end{cases}    
\end{align}
The proof feature set $\absFeature_{S}$ is sufficient iff the property $(\inreg, \outprop)$ can be verified by $\ver$ on $N$ with the pruned weight matrix $W_{l}(S)$.
As described in Section~\ref{sec:prelims}, for property verification the verifier computes $\ver$ an over-approximated output region $\abstraction$ of $N$ over the input region $\phi$.
Given that we never change the input region $\inreg$ and the proof feature extractor composed of the first $l - 1$ layers of $N$ and the verifier $\ver$, the output region $\abstraction$ only depends on the pruning done at the final layer.   
Now let $\abstraction(W_{l}, S)$ denote the over-approximated output region corresponding to $W_{l}(S)$. 
The neural network $N$ can be verified by $\ver$ on the property $(\inreg, \outprop)$ with $W_{l}(S)$ iff the lower bound  $\lb(\abstraction(W_{l}, S)) \geq 0$.
Therefore, finding $S_{0}$ corresponding to a minimum proof feature set $\cannFeature$ can be formulated as below where for any $S \subseteq [d_{l-1}]$, $|S|$ denotes the number of elements in $S$. 
\begin{align}
    \argmin_{S \neq \emptyset,\; S \subseteq [d_{l-1}]} |S| \;\;\;\text{s.t.}\;\;\lb(\abstraction(W_{l}, S)) \geq 0 
\end{align}

\subsection{Approximate Minimum Proof Feature Extraction}
\label{sec:minFeatureExtraction}
The search space for finding $\cannFeature$ is prohibitively large and contains $2^{d_{l-1}}$ possible candidates. So, computing a minimum solution with an exhaustive search is infeasible.
Even checking the sufficiency of any arbitrary proof feature set $\suffFeature$ (Definition~\ref{def:sufficent}) is not trivial and requires expensive verifier invocations. We note that even making $O(d_{l-1})$ verifier calls is too expensive for the network sizes considered in our evaluation.
Given the large DNN size, exponential search space, and high verifier cost, efficiently computing a \textit{minimum} sufficient proof feature set is computationally intractable.
We design a practically efficient approximation algorithm based on a greedy heuristic that can generate a smaller (may not always be minimum) sufficient feature set with only $O(\log(d_{l-1}))$ verifier calls.  
At a high level, for each proof feature $\absFeature_{n_i}$ contained in a sufficient feature set, the heuristic tries to estimate whether pruning $\absFeature_{n_i}$ violates the property $(\inreg, \outprop)$ or not. 
Subsequently, we prioritize pruning of those proof features $\absFeature_{n_i}$ that, if pruned, will likely preserve the proof of the property ($\inreg$,$\outprop$) with the verifier $\ver$.


For any proof feature $\absFeature_{n_i} \in \suffFeature$ where $\suffFeature$ is sufficient and proves the property $(\phi, \psi)$, we estimate the change $\contri(\absFeature_{n_i}, \suffFeature)$ that occurs to $\lb(\abstraction(W_{l}, S))$ if $\absFeature_{n_i}$ is pruned from $\suffFeature$.
Let, the over-approximated output region computed by verifier $\ver$ corresponding to $\suffFeature \setminus \{\absFeature_{n_i}\}$ be $\abstraction(W_{l}, S \setminus \{i\})$ then $\contri(\absFeature_{n_i}, \suffFeature)$ is defined as follows
\begin{equation*}
\contri(\absFeature_{n_i}, \suffFeature) = |\lb(\abstraction(W_{l}, S)) - \lb(\abstraction(W_{l}, S \setminus \{i\}))|
\end{equation*}
Intuitively, proof features $\absFeature_{n_i}$ with higher values of $\contri(\absFeature_{n_i}, \suffFeature)$ for some sufficient feature set $\suffFeature \subseteq \absFeature$ are responsible for large changes to $\lb(\abstraction(W_{l}(S)))$ and likely to break the proof if pruned.
Note, $\contri(\absFeature_{n_i}, \suffFeature)$ depends on the particular sufficient proof set $\suffFeature$ and does not estimate the global importance of $\absFeature_{n_i}$ independent of the choice of $\suffFeature$.
To mitigate this issue, while defining the priority $P(\absFeature_{n_i})$ of a proof feature $\absFeature_{n_i}$ we take the maximum of $\contri(\absFeature_{n_i}, \suffFeature)$ across all sufficient feature set $\suffFeature$ containing $\absFeature_{n_i}$. 
Let, $\fset{\absFeature_{n_i}}$ denote set of all sufficient $\suffFeature$ containing $\absFeature_{n_i}$. 
Then, $P(\absFeature_{n_i})$ can be formally defined as follows
\begin{equation}
P(\absFeature_{n_{i}}) = \max_{\suffFeature \in \fset{\absFeature_{n_i}}} \contri(\absFeature_{n_{i}}, \suffFeature)
\end{equation}
Given the set $\fset{\absFeature_{n_i}}$ can be exponentially large,
finding the maximum value of $\contri(\absFeature_{n_i}, \suffFeature)$ over $\fset{\absFeature_{n_i}}$ is practically infeasible.
Instead, we compute a resonably tight upper bound $P_{ub}(\absFeature_{n_i})$ on  $P(\absFeature_{n_i})$ by estimating a global upper bound on $\contri(\absFeature_{n_i}, \suffFeature)$, that holds $\forall \suffFeature \in \fset{\absFeature_{n_i}}$.
The proposed upper bound is independent of the choice of $\absFeature_{S} \in \fset{\absFeature_{n_i}}$ and therefore removes the need to iterate over $\fset{\absFeature_{n_i}}$ enabling efficient computation.   
For the network $N$ and input region $\phi$, let $\abstraction_{l-1}$ denote the over-approximate symbolic region computed by $\ver$ at the penultimate layer.
Then $\forall \suffFeature \in \fset{\absFeature_{n_i}}$ the global uppper bound of $\contri(\absFeature_{n_{i}}, \suffFeature)$ can be computed as follows where for any vector $X \in \mathbb{R}^{d_{l-1}}$, $x_i$ denotes its $i$-th coordinate:
\begin{align*}
 \contri(\absFeature_{n_{i}}, \suffFeature) &\leq  \max_{X \in \abstraction_{l-1}} |(C^TW_{l}(S)X - C^TW_{l}(S \setminus \{i\})X)| \\
 &= \max_{X \in \abstraction_{l-1}} |(C^T W_{l}[:,i]) \cdot x_{i})| \\
 P(\absFeature_{n_i}) &\leq \max_{X \in \abstraction_{l-1}} |(C^T W_{l}[:,i]) \cdot x_{i})|
\end{align*}
Now, any proof feature $\absFeature_{n_i} = [l_{n_i}, u_{n_i}]$ computed by $\ver$ contains all possible values of $x_i$ where $X \in \abstraction_{l-1}$. 
Leveraging this observation, we can further simplify the upper bound  $P_{ub}(\absFeature_{n_i})$ of $P(\absFeature_{n_i})$ as shown below.
\begin{gather}
    P(\absFeature_{n_i}) \leq \max_{x_i \in [l_{n_i}, u_{n_i}]} |(C^T W_{l}[:,i])| \cdot x_{i})| \notag \\
 P_{ub}(\absFeature_{n_i}) = |(C^T W_{l}[:,i])| \cdot \max(|l_{n_i}|, |u_{n_i}|) \label{eq:apxpriority}
\end{gather}
This simplification ensures that $P_{ub}(\absFeature_{n_i})$ for all $\absFeature_{n_i}$ can be computed with $O(d_{l-1})$ elementary vector operations and a single verifier call that computes the intervals $[l_{n_i}, u_{n_i}]$.
Next, we describe how we compute an approximate feature set using the feature priorities $P_{ub}(\absFeature_{n_i})$. 
For any feature $\absFeature_{n_i}$, $P_{ub}(\absFeature_{n_i})$ estimates the importance of $\absFeature_{n_i}$ in preserving the proof.
So, a trivial step is to just prune all the proof features from $\absFeature$ whose $P_{ub}$ is 0. These features do not have any contribution to the proof of the property $(\inreg, \outprop)$ by the verifier $\ver$. 
This step forms a trivial algorithm. 
However, this is not enough. We can further prune some more proof features leading to a yet smaller set.  For this, we propose an iterative algorithm \textbf{\algname} shown in Algorithm~\ref{alg:featureComputation} ($\alg$) which
maintains two set namely, $\absFeature^{(\alg)}_{S_0}$ and $\absFeature^{(\alg)}_{S}$. $\absFeature^{(\alg)}_{S_0}$ contains the features guaranteed to be included in the final answer computed by \algname and $\absFeature^{(\alg)}_{S}$ contains the candidate features to be pruned by the algorithm. 
At every step, the algorithm ensures that the set $\absFeature^{(\alg)}_{S} \cup \absFeature^{(\alg)}_{S_0}$ is sufficent and iteratively reduces its size by pruning proof features from $\absFeature^{(\alg)}_{S}$.
The algorithm iteratively prunes the feature $\absFeature_{n_i}$ with the lowest value of $P_{ub}(\absFeature_{n_i})$ from $\absFeature^{(\alg)}_{S}$ to maximize the likelihood that $\absFeature^{(\alg)}_{S} \cup \absFeature^{(\alg)}_{S_0}$ remains sufficient at each step.
At Line 8 in the algorithm,  $\absFeature^{(\alg)}_{S_0}$ and $\absFeature^{(\alg)}_{S}$ are initialized as $\{\}$ (empty set) and $\absFeature$ respectively. 
%
Removing a single feature in each iteration and checking the sufficiency of the remaining features in the worst case leads to $O(d_{l-1})$ verifier calls which are infeasible.
Instead, at each step, from $\absFeature^{(\alg)}_{S}$ our algorithm greedily picks top-$|S|/2$ features (line 10) $\absFeature^{(\alg)}_{S_1}$ based on their priority and invokes the verifier $\ver$ to check the sufficiency of $\absFeature^{(\alg)}_{S_0} \cup \absFeature^{(\alg)}_{S_1}$ (line 12). 
If the feature set $\absFeature^{(\alg)}_{S_0} \cup \absFeature^{(\alg)}_{S_1}$ is sufficient (line 13), $\alg$ removes all features in $\absFeature^{(\alg)}_{S} \setminus \absFeature^{(\alg)}_{S_1}$ from $\absFeature^{(\alg)}_{S}$ and therefore  $\absFeature^{(\alg)}_{S}$ is updated as $\absFeature^{(\alg)}_{S_1}$ in this step (line 14).
Otherwise, if $\absFeature^{(\alg)}_{S_0} \cup \absFeature^{(\alg)}_{S_1}$ does not preserve the property ($\inreg$,$\outprop$) (line 15), $\alg$ adds all feature in $\absFeature^{(\alg)}_{S_1}$ to $\absFeature^{(\alg)}_{S_0}$ (line 16) and replaces $\absFeature^{(\alg)}_{S}$ with $\absFeature^{(\alg)}_{S} \setminus \absFeature^{(\alg)}_{S_1}$ (line 17).
The algorithm ($\alg$) terminates after all features in $\absFeature^{(\alg)}_{S}$ are exhausted. 
Since at every step, the algorithm reduces size of $\absFeature^{(\alg)}_{S}$ by half, it always terminates within $O(\log(d_{l-1}))$ verifier calls.
\begin{algorithm}[tb]
   \caption{Approx. minimum proof feature computation}
   \label{alg:featureComputation}
\begin{algorithmic}[1]
   \STATE {\bfseries Input:}  network $N$, property $(\inreg, \outprop)$, verifier $\ver$.
   \STATE {\bfseries Output:}  approximate minimal proof features $\absFeature^{(\alg)}_{S_0}$,
   \IF{ $\ver$ can not verify $N$ on $(\inreg, \outprop)$}
        \STATE \textbf{return}
    \ENDIF
   \STATE \text{Calculate all proof features for input region $\inreg$.}
   \STATE \text{Calculate priority $P_{ub}(\absFeature_{n_i})$ all proof features $\absFeature_{n_i}$.}   
   \STATE \textbf{Initialization: } $\absFeature^{(\alg)}_{S_0} = \{\}$,\; $\absFeature^{(\alg)}_{S} = \absFeature$ \label{algline:ini}
   \WHILE{$\absFeature^{(\alg)}_{S}$ is not empty}
    \STATE $\absFeature^{(\alg)}_{S_1} = \text{top-$|S|/2$ features selected based on $P_{ub}(\absFeature_{n_i})$}$ \label{algline:topFeature}
    \STATE $\absFeature^{(\alg)}_{S_2} = \absFeature^{(\alg)}_{S} \setminus \absFeature^{(\alg)}_{S_1}$
    \STATE \text{Check sufficiency of $\absFeature^{(\alg)}_{S_0} \cup \absFeature^{(\alg)}_{S_1}$ with $\ver$ on $(\inreg, \outprop)$}
    \IF{ $\absFeature^{(\alg)}_{S_0} \cup \absFeature^{(\alg)}_{S_1}$ is sufficient}
    \STATE $\absFeature^{(\alg)}_{S} = \absFeature^{(\alg)}_{S_1}$ \hfill\COMMENT{all features in $\absFeature_{S_2}
    $ are pruned}
    \ELSE
    \STATE $\absFeature^{(\alg)}_{S_0} = \absFeature^{(\alg)}_{S_0} \cup \absFeature^{(\alg)}_{S_1}$
    \STATE $\absFeature^{(\alg)}_{S} = \absFeature^{(\alg)}_{S_2}$
    \ENDIF
   \ENDWHILE
\STATE \textbf{return} proof features $\absFeature^{(\alg)}_{S_0}$.
\end{algorithmic}
\end{algorithm}

\textbf{Limitations.} We note that the scalability of our method is ultimately limited by the scalability of the existing verifiers. Therefore, \algname currently cannot handle networks for larger datasets like ImageNet. Nonetheless, \algname is general and compatible with any verification algorithm. Therefore, \algname will benefit from any future advances to enable the neural network verifiers to scale to larger datasets.

Next, we derive mathematical guarantees about the correctness and efficacy of Algorithm~\ref{alg:featureComputation}. 
For correctness, we prove that the feature set $\absFeature^{(\alg)}_{S_0}$ is always sufficient (Definition \ref{def:sufficent}).
For efficacy,  we theoretically find a non-trivial upper bound on the size of $\absFeature^{(\alg)}_{S}$. 

\begin{theorem}
\label{thm:sufficiencyProof}
If the verifier $\ver$ can prove the property $(\inreg,\outprop)$ on the network $N$, then $\cannFeature^{(\alg)}$ computed by Algorithm \ref{alg:featureComputation} is sufficient (Definition \ref{def:sufficent}). 
\end{theorem}
This follows from the fact that \algname Algorithm ensures at each step that $\absFeature^{\mathbb{(A)}}_{S_0} \cup \absFeature^{\mathbb{(A)}}_{S}$ is sufficient.  
Hence, at termination the feature set $\absFeature^{(\alg)}_{S_0}$ is sufficient.
The complete proof of Theorem~\ref{thm:sufficiencyProof} is in appendix~\ref{appendix:proofs}.
Next, we find a non-trivial upper bound on the size of $\absFeature^{(\alg)}_{S}$ computed by the algorithm. 
\vspace{-3mm}
\begin{definition}
\label{def:zeroproofFeatures}
For $\absFeature$, zero proof features set $Z(\absFeature)$ denotes the  proof features $\absFeature_{n_i} \in \absFeature$ with $P_{ub}(\absFeature_{n_i}) = 0$.
\end{definition}
\vspace{-3mm}
Note, any proof feature $\absFeature_{n_i} \in Z(\absFeature)$ can be trivially removed without breaking the proof.
Further, we show that some additional proof features will be filtered out from the original proof feature set.
So, the size of the proof feature set $\cannFeature^{(\alg)}$ extracted by \algname is guaranteed to be less than the value computed in  Theorem~\ref{thm:efficacyBound}.
\begin{table*}[htbp]
\centering
\captionsetup{justification=centering}
\caption{\algname Efficacy Analysis}
\label{table:compareoriginal}
\resizebox{0.98\textwidth}{!}{
\begin{tabular}{@{}l l l l l c c c c r r@{}}
\toprule
Dataset & Training & Input & No. of &Original & \multicolumn{2}{c}{Proof} & \multicolumn{2}{c}{Proof} & No. of proofs & No. of proofs \\ 
\text{ } & Method & Region ($\phi$) & proved & Feature  & \multicolumn{2}{c}{Feature Count} & \multicolumn{2}{c}{Feature Count} & with $\leq 5$ & with $\leq 10$ \\
\text{ } & \text{} &  \text{eps ($\epsilon$)} & properties & Count & \multicolumn{2}{c}{(Baseline)} & \multicolumn{2}{c}{(\algname)} & proof features & proof features \\ 
\text{ } & \text{} & \text{} & \text{} &\text{} & Mean & Median & Mean & Median & (\algname) & (\algname) \\ 
\midrule
\textbf{MNIST} &Standard & 0.02 & 297 & 256 & 23.19 & 19 & 2.23 & 2 & 291 & 297 \\
 \text{} & PGD Trained  & 0.02 & 410 &1000 & 218.02 & 218 & 5.57 & 3 & 317 & 365 \\
\text{} & COLT & 0.02 & 447 &250 & 44.43 & 45 & 7.37 & 6 & 217 & 351 \\
\text{ } & CROWN-IBP & 0.02 & 482 & 128 & 42.38 & 43 & 5.84 & 4 & 331 & 400  \\
\midrule
\textbf{MNIST}  & PGD Trained  & 0.1 & 163 & 1000 & 279.31 & 278 & 5.29 & 3 & 131
 & 149 \\
\text{} & COLT & 0.1 & 215 & 250 & 51.01 & 51 & 5.97 & 5 & 133 & 203 \\
\text{ } & CROWN-IBP & 0.1 & 410 & 128 & 47.92 & 48 & 5.86 & 4 & 267 & 343 \\
\midrule
\textbf{CIFAR-10} &Standard & $0.2/255$ & 255 & 100 & 52.93 & 53 & 10.38 & 7 & 120  & 164 \\
\text{} & PGD Trained  & $0.2/255$ & 235 & 100 & 54.29 & 54 & 8.04 & 3 & 155 & 177 \\
\text{} & COLT & $0.2/255$ & 265 & 250 & 77.71 & 78 & 9.12 & 4 & 148 & 192 \\
\text{ } & CROWN-IBP & $0.2/255$ & 265 & 256 & 20.23 & 21 & 5.30 & 3 & 179 & 222 \\
\midrule
\textbf{CIFAR-10} & PGD Trained  & $2/255$ & 133 & 100 & 108 & 65 & 7.06 & 3 & 108 & 118 \\
\text{} & COLT & $2/255$ & 228 & 250 & 86.62 & 86 & 8.65 & 4 & 127 & 160  \\
\text{ } & CROWN-IBP & $2/255$ & 188 & 256 & 23.03 & 23 & 4.31 & 3 & 140 & 173 \\
\midrule
\end{tabular}}
\vspace{-4mm}
\end{table*}
\begin{theorem}
\label{thm:efficacyBound}
Let, $P_{max}$ denote the maximum of all priorities $P_{ub}(\absFeature_{n_i})$ over $\absFeature$.
Given any network $N$ is verified on $(\inreg, \outprop)$ with verifier $\ver$ then $|\cannFeature^{(\alg)}|\leq d_{l-1} - |Z(\absFeature)| - \lfloor \frac{\lb(\abstraction)}{P_{max}} \rfloor $
\end{theorem}
The exact proof for Theorem~\ref{thm:efficacyBound} can be found in Appendix~\ref{appendix:proofs}
\subsection{Interpreting proof features}
\label{sec:featuresInvestigation}
For interpreting proofs of DNN robustness, we now develop methods to analyze the semantic meaning of the extracted proof features.
There exists a plethora of works that compute local DNN explanations \cite{sundararajan2017axiomatic, DBLP:journals/corr/SmilkovTKVW17}.
However, these techniques are insufficient to generate an explanation w.r.t an input region.
To mitigate this, we adapt the existing local visualization techniques for visualizing the extracted proof features.  
Given a proof feature $\absFeature_{n_i}$, we intend to compute $\grad(\absFeature_{n_i}, \phi) = \expect_{X \sim \inreg} \grad(n_i, X)$ which is the mean gradient of the output of $n_i$ w.r.t the inputs from $\inreg$.
For each input dimension (pixel in case of images) $j \in [d_{0}]$ the $j$-th component of $\grad(\absFeature_{n_i}, \inreg)$ estimates its relevance w.r.t proof feature $\absFeature_{n_i}$ - the higher is the gradient value, the higher is its relevance.
Considering that the input region $\inreg$ contains infinitely many inputs, exactly computing $\grad(\absFeature_{n_i}, \inreg)$ is impossible. Rather, we statistically estimate $\grad(\absFeature_{n_i}, \inreg)$ by a resonably large sample $X_{S}$ drawn uniformly from $\inreg$. 


%% file: experimentalEval.tex
\section{Experimental Evaluation}
\subsection{Experimental setup} 
For evaluation we use convolutional networks trained on two popular datasets - MNIST \cite{mnist89} CIFAR-10 \cite{krizhevsky2009learning} shown in Table~\ref{table:compareoriginal}.  The networks are trained with standard training and three state-of-the-art robust training methodologies - adversarial training (PGD training) \cite{madry:17}, certified robust training (CROWN-IBP) \cite{ZhangCXGSLBH20} and a combination of both adversarial and certified training (COLT) \cite{BalunovicV:20}.
For our experiments, we use pre-trained publically available networks - the standard and PGD-trained networks are taken from the ERAN project \cite{singh2018robustness}, COLT-trained networks from COLT website \cite{BalunovicV:20}, and CROWN-IBP trained networks from the CROWN-IBP repository \cite{ZhangCXGSLBH20}.  
Similar to most of the existing works on neural network verification~\cite{carlini2017towards, singh2018robustness}, we use $L_{\infty}$-based local robustness properties.
Here, the input region $\inreg$ contains all images obtained by perturbing the intensity of each pixel in the input image independently within a bound $\epsilon \in \mathbb{R}$. $\outprop$ specifies a region where the network output for the correct class is higher than all other classes. 
We use $\epsilon_{train} = 0.3$ for all robustly trained MNIST networks and $\epsilon_{train} = 8/255$ for all robustly trained CIFAR-10 networks.
Unless specified otherwise, the proofs are generated by running the popular DeepZ~\cite{singh2018robustness} verifier.
We perform all our experiments on a 16-core 12th-gen i7 machine with 16 GB of RAM. 

\subsection{Efficacy of \algname Algorithm}
In this section, we experimentally evaluate the efficacy of the \algname based on the size of the output sufficient feature sets.
Given that there is no existing work for pruning proof feature sets, we use the upper bound computed in Theorem~\ref{thm:efficacyBound} as the baseline.
Note that this bound 
\begin{figure}[h]
    \centering
    \begin{subfigure}[b]{0.5\textwidth}
         \centering
         \includegraphics[width=0.9\textwidth, height=0.45\textwidth]{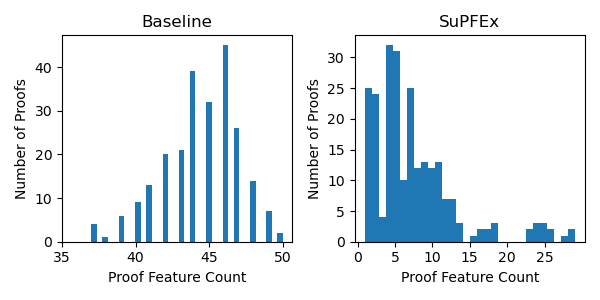}
         \vspace{-2mm}
         \caption{For MNIST Dataset}
         \label{fig:mnisthisto}
     \end{subfigure}
     \hfill
     \vspace{1mm}
     \begin{subfigure}[b]{0.5\textwidth}
         \centering
         \includegraphics[width=0.9\textwidth, height=0.45\textwidth]{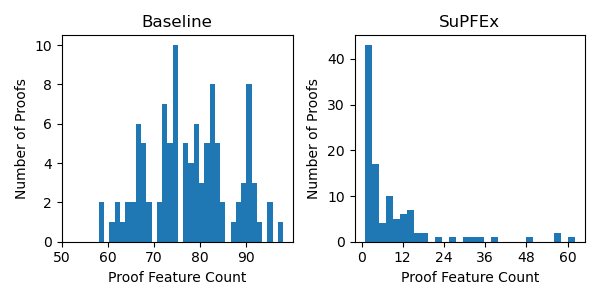}
         \vspace{-2mm}
         \caption{For CIFAR-10 dataset}
         \label{fig:cifarhisto}
     \end{subfigure}
        \caption{Distribution of the size of the  proof feature set computed by \algname Algorithm on COLT-trained networks.}
         \vspace{-2mm}
        \label{fig:histo}
\end{figure}
is better than the size of the proof feature set extracted by the trivial algorithm - one that only removes only ``zero"  features which include the proof features ($[l, u]$) where both $l = u = 0$. (Definition~\ref{def:zeroproofFeatures})  
For each network, we use $500$ randomly picked images from their corresponding test sets. The $\epsilon$ used for MNIST networks is $0.02$ and that for CIFAR-10 networks is $0.2 / 255$. 
We note that although the robustly trained networks can be verified robust for higher values of $\epsilon$, it is not possible to verify standard networks with such high values.
To achieve common ground, we use small $\epsilon$ values for experiments involving standard networks and conduct separate experiments on only robustly trained networks with higher values of $\epsilon$ ($0.1$ for MNIST, $2/255$ for CIFAR-10 networks).
As shown in Table~\ref{table:compareoriginal} we do not observe any significant change in the performance of \algname w.r.t different $\epsilon$-values. 
\\
In Table~\ref{table:compareoriginal}, we show 
the value of $\epsilon$ used to define the region $\inreg$ in column 3, and the total number of properties proved out of 500 in column 4.
The size of the original proof feature size corresponding to each network is shown in column 5, the mean and median of the proof feature set size computed using Theorem~\ref{thm:efficacyBound} in columns 6 and 7 respectively, and the mean and median of the proof feature set size computed using \algname in columns 8 and 9 respectively.
We note that feature sets obtained by \algname are significantly smaller than the upper bound provided by Theorem~\ref{thm:efficacyBound}. For example, in the case of the PGD trained MNIST network with $1000$ neurons in the penultimate layer, the average size computed from Theorem~\ref{thm:efficacyBound} is $218.02$, while that obtained using \algname is only $5.57$. 
In the last two columns of Table~\ref{table:compareoriginal}, we summarise the percentage of cases where we are able to achieve a proof feature set of size less than or equal to $5$ and $10$ respectively. Figures~\ref{fig:mnisthisto} and~\ref{fig:cifarhisto} display a histogram where the x-axis is the size of the extracted proof feature set using \algname and y-axis is the number of local robustness properties for COLT-trained DNNs. Histograms for other DNNs are presented in Appendix~\ref{appendix:results}. These histograms are skewed towards the left which means that for most of the local properties, we are able to generate a small set of proof features using \algname.
\begin{figure}[h!]
    \centering
    \begin{subfigure}[b]{0.5\textwidth}
         \centering
         \includegraphics[width=0.8\textwidth, height=0.5\textwidth]{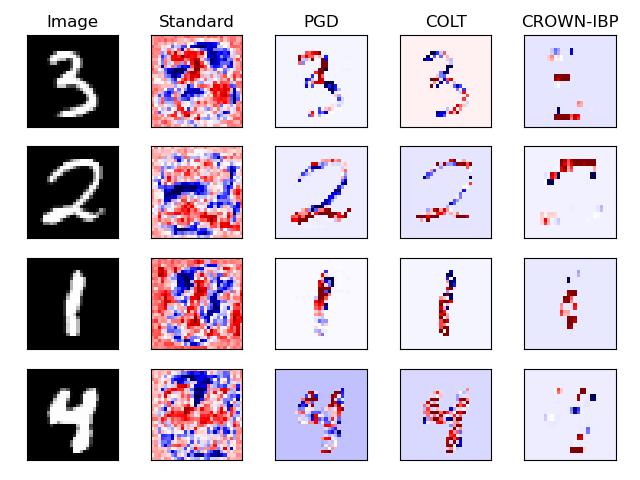}
         \vspace{-2mm}
         \caption{Gradient maps generated on MNIST networks.}
         \label{fig:mnist_grad_map}
     \end{subfigure}
     \hfill
     \vspace{0.5mm}
     \begin{subfigure}[b]{0.5\textwidth}
         \centering
         \includegraphics[width=0.8\textwidth, height=0.5\textwidth]{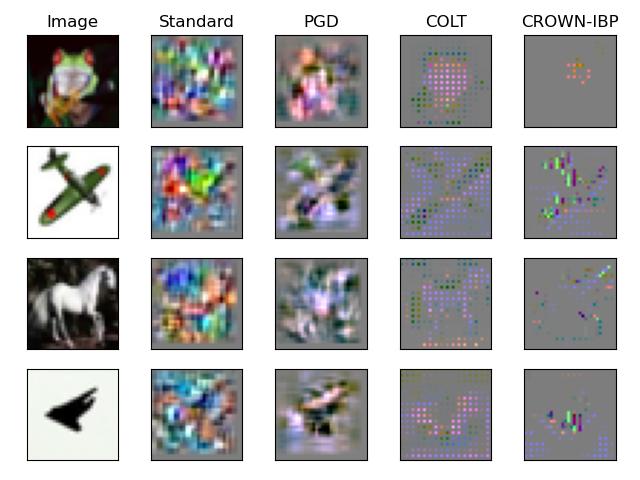}
         \vspace{-2mm}
         \caption{Gradient maps generated on CIFAR-10 networks.}
         \label{fig:cifar_grad_map}
     \end{subfigure}
\caption{
The top proof feature corresponding to DNNs trained using different methods rely on different input features.
}
         \vspace{-2mm}
\label{fig:grad_maps}
\end{figure}
\subsection{Qualititive comparison of robustness proofs}
\label{sec:qualitativeAnalysis}
It has been observed in \cite{tsipras2018robustness} that the standardly trained networks rely on some of the spurious features in the input in order to gain a higher accuracy and as a result, are not very robust against adversarial attacks. On the other hand, the empirically robustly trained networks rely more on human-understandable features and are, therefore, more robust against attacks. This empirical robustness comes at cost of reduced accuracy. 
So, there is an inherent dissimilarity between the types of input features that the standard and adversarially trained networks rely on while classifying a single input. Also, certified robust trained networks are even more robust than the empirically trained ones, however, they report even less accuracy~\cite{MLSYS2021_ca46c1b9}. 
In this section, we interpret proof features obtained with \algname and use these interpretations to qualitatively check whether the dissimilarities are also evident in the invariants captured by the different proofs of the same robustness property on standard and robustly trained networks.
We also study the effect of certified robust training methods like CROWN-IBP \cite{ZhangCXGSLBH20}, empirically robust training methods like PGD \cite{madry:17} and training methods that combine both  adversarial and certified training like COLT \cite{BalunovicV:20} on the proof features. \\
For a local input region $\inreg$, we say that a robustness proof is semantically meaningful if it focuses on the relevant features of the output class for images contained inside $\inreg$ and not on the spurious features. In the case of MNIST or CIFAR-10 images, spurious features are the pixels that form a part of the background of the image, whereas important features are the pixels that are a part of the actual object being identified by the network.
Gradient map of the extracted proof features w.r.t. to the input region $\inreg$ gives us an idea of the input pixels that the network focuses on. We obtain the gradient maps by calculating the mean gradient over 100 uniformly drawn samples from $\inreg$ as described in Section~\ref{sec:featuresInvestigation}.
As done in \cite{tsipras2018robustness}, to avoid introducing any inherent bias in proof feature visualization, no preprocessing (other than scaling and clipping for visualization) is applied to the gradients obtained for each individual sample.
\\
In Fig.~\ref{fig:grad_maps}, we compare the gradient maps corresponding to the top proof feature (the one having the highest priority $P_{ub}(\absFeature_{n_i})$) on networks from Table~\ref{table:compareoriginal} on representative images of different output classes in the MNIST and CIFAR10 test sets. 
The experiments leads us to interesting observations - even if some property is verified for both the standard network and the robustly trained network, there is a difference in the human interpretability of the types of input features that the proofs rely on.
The standard networks and the provably robust trained networks like CROWN-IBP are the two extremes of the spectrum. 
For the networks obtained with standard training, we observe that although the top-proof feature does depend on some of the semantically meaningful regions of the input image, the gradient at several spurious features is also non-zero.
On the other hand, the top proof feature corresponding to state-of-the-art provably robust training method CROWN-IBP filters out most of the spurious features, but it also misses out on some meaningful features.
The proofs of PGD-trained networks filter out the spurious features and are, therefore, more semantically aligned than the standard networks. 
The proofs of the training methods that combine both empirical robustness and provable robustness like COLT in a way provide the best of both worlds by not only selectively filtering out the spurious features but also highlighting the more human interpretable features, unlike the certifiably trained networks. 
So, as the training methods tend to regularize more for robustness, their proofs become more conservative in relying on the input features. 
To further support our observation, we show additional plots for the top proof feature visualization in Appendix~\ref{sec:addiproofFeatureVis} and visualization for multiple proof features in Appendix~\ref{sec:addiMultiFeatureVis}.
We also conduct experiments for different values of $\epsilon$ used for defining $\inreg$. 
The extracted proof features set w.r.t high $\epsilon$ values ($\epsilon = 0.1$ for MNIST and $\epsilon = 2/255$ for CIFAR-10) are similar to those generated with smaller $\epsilon$. 
The gradient maps corresponding to the top feature for higher $\epsilon$ values are also similar as shown in Appendix~\ref{sec:addiHighEpsPlots}.
For COLT-trained MNIST networks, in
~\ref{sec:topBottomFeature} we compare the gradients of top proof features retained by \algname with the pruned proof features with low priority.
As expected, graidents of the pruned proof features with low priority contain spurious input features.
\\
\begin{figure}[t]
\centering
\includegraphics[height = 0.27\textwidth]{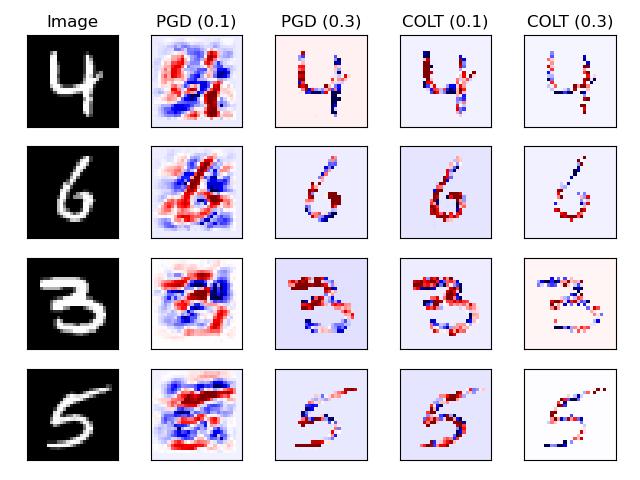}
\vspace{-2mm}
\caption{Visualization of gradients of the top proof feature for PGD and COLT networks trained
using different values of $\epsilon_{train} $.}
\vspace{-3mm}
\label{fig:grad_maps_with_changed_eps}
\end{figure}

\subsection{Sensitivity analysis on training parameters}
It is expected that DNNs trained with larger  $\epsilon_{train}$ values are more robust. So, we analyze the sensitivity of the extracted proof features to $\epsilon_{train}$.
We use the DNNs trained with PGD and COLT and $\epsilon_{train} \in \{0.1,0.3\}$ on the MNIST dataset. 
Fig~\ref{fig:grad_maps_with_changed_eps}  visualize proof features for the DNNs with
additional plots are available in Appendix~\ref{sec:addSensitivity}.
We observe that by increasing the value of $\epsilon_{train}$, the top proof feature filters out more input features. This is aligned with our observation in Section~\ref{sec:qualitativeAnalysis} that a more robustly trained neural networks are more conservative in using the input features.

\subsection{Comparing proofs of different verifiers}

The proof features extracted by \algname are specific to the proof generated by the verifier. 
In this experiment, we compare proof features generated by two popular verifiers IBP \cite{ZhangCXGSLBH20, IBP} and DeepZ on networks shown in Table~\ref{table:compareoriginal} for the same properties as before.  
Note that, although IBP is computationally efficient, it is less precise than DeepZ. 
For standard DNNs, most of the properties cannot be proved by IBP. Hence, 
in this experiment, we omit standard DNNs and also, consider only the properties that can be verified by both DeepZ and IBP. 
Table~\ref{tab:diff_ver} 
presents the \% cases where the top proof feature computed by both the verifiers is the same (column 2), the \% cases where the top 5 proof features computed by both the verifiers are the same and the \% cases where the complete proof feature sets computed by both the verifiers are same. We observe that for the MNIST dataset, in 100\% of the cases for PGD-trained and COLT-trained networks and in 99.79\% cases for the CROWN-IBP trained networks, the top feature computed by both the verifiers is the same. Detailed table is available in Appendix~\ref{sec:additionalVerRes}.

%% file: conclusion.tex
\vspace{-2mm}
\section{Conclusion}
In this work, we develop a novel method called \algname to interpret neural network robustness proofs. 
We empirically establish that even if a property holds for a DNN, the proof for the property may rely on spurious or semantically meaningful features depending on the training method used to train the DNNs. We believe that \algname can be applied for diagnosing the trustworthiness of DNNs inside their development pipeline.  

\begin{table}[b]
    \centering
    \vspace{-3mm}
    \captionsetup{justification=centering}
    \caption{Comparing proofs of IBP \& DeepZ}
    \resizebox{0.5\textwidth}{!}{
    \begin{tabular}{@{}l c c c c r r@{}}
    \toprule
    Training & \multicolumn{2}{c}{\% proofs with the} & \multicolumn{2}{c}{\% proofs with the} & \multicolumn{2}{r}{\% proofs with the} \\
    Method & \multicolumn{2}{c}{same top feature} & \multicolumn{2}{c}{same top-5 feature} & \multicolumn{2}{r}{same feature set} \\
    &MNIST & CIFAR10 &MNIST & CIFAR10 &MNIST & CIFAR10 \\
    \midrule
    PGD Trained & 100 \% & 100 \% & 92.0 \% & 98.31 \% & 92.0 \% & 96.87 \%\\
    COLT & 100 \% & 97.87 \% & 87.17 \% & 92.53 \% & 82.05 \% & 89.36 \% \\
    CROWN-IBP & 99.79 \% & 100 \% & 96.26 \% & 97.92 \% & 93.15 \% & 95.89 \%\\
    \midrule
    \end{tabular}}
    \label{tab:diff_ver}
\end{table}

%% file: proofs.tex
\section{Mathematical proofs}
\label{appendix:proofs}
\textbf{Theorem}~\ref{thm:sufficiencyProof}: If the verifier $\ver$ can prove the property $(\inreg,\outprop)$ on the network $N$, then $\cannFeature^{(\alg)}$ computed by Algorithm \ref{alg:featureComputation} is sufficient (Definition \ref{def:sufficent}). 
\begin{proof}By induction on the number of steps of the while loop.\\ 
\textbf{Induction Hypothesis: }At each step of the loop, $\absFeature^{\mathbb{(A)}}_{S_0} \cup \absFeature^{\mathbb{(A)}}_{S}$ is sufficient.\\
\textbf{Base Case: }At step 0, i.e., at initialization, $\absFeature^{\mathbb{(A)}}_{S_0} = \{\}$ and $\absFeature^{\mathbb{(A)}}_{S}\absFeature$. So, $\absFeature^{\mathbb{(A)}}_{S_0} \cup \absFeature^{\mathbb{(A)}}_{S} = \absFeature$. Given that $\ver$ proves the property $(\inreg, \outprop)$ on $N$, from Definition~\ref{def:sufficent}, $\absFeature$ is sufficient.\\
\textbf{Induction Case: }Let $\absFeature^{\mathbb{(A)}}_{S_0} \cup \absFeature^{\mathbb{(A)}}_{S}$ be sufficient for $n$-th step of the loop. Consider the following cases for $(n+1)$-th step of the loop.
\begin{enumerate}
    \item Let $\absFeature^{\mathbb{(A)}}_{S_0} \cup \absFeature^{\mathbb{(A)}}_{S_1}$ be sufficient at line 12. In this case, $\absFeature^{\mathbb{{A}}}_{S}$ is updated by $\absFeature^{\mathbb{(A)}}_{S_1}$ (line 14). So, $\absFeature^{\mathbb{(A)}}_{S_0} \cup \absFeature^{\mathbb{(A)}}_{S}$ is sufficient.
    \item Let $\absFeature^{\mathbb{(A)}}_{S_0} \cup \absFeature^{\mathbb{(A)}}_{S_1}$ be not sufficient at line 12. In this case, $\absFeature^{\mathbb{(A)}}_{S_0} $ and $ \absFeature^{\mathbb{(A)}}_{S}$ are updated as in lines 16 and 17. Let the new $\absFeature^{\mathbb{(A)}}_{S_0} $ and $ \absFeature^{\mathbb{(A)}}_{S}$ be $\absFeature'^{\mathbb{(A)}}_{S_0} $ and $ \absFeature'^{\mathbb{(A)}}_{S}$. So, $\absFeature'^{\mathbb{(A)}}_{S_0} = \absFeature^{\mathbb{(A)}}_{S_0} \cup \absFeature^{\mathbb{(A)}}_{S_1}$ and $\absFeature'^{\mathbb{(A)}}_{S} = \absFeature^{\mathbb{(A)}}_{S_2}$. So, $\absFeature'^{\mathbb{(A)}}_{S_0} \cup \absFeature'^{\mathbb{(A)}}_{S} = \absFeature^{\mathbb{(A)}}_{S_0} \cup \absFeature^{\mathbb{(A)}}_{S_1} \cup \absFeature^{\mathbb{(A)}}_{S_2}$. Also, $\absFeature^{\mathbb{(A)}}_{S_1} \cup \absFeature^{\mathbb{(A)}}_{S_2} = \absFeature^{\mathbb{(A)}}_{S}$. So, $\absFeature'^{\mathbb{(A)}}_{S_0} \cup \absFeature'^{\mathbb{(A)}}_{S} = \absFeature^{\mathbb{(A)}}_{S_0} \cup \absFeature^{\mathbb{(A)}}_{S}$. So, from induction hypothesis, $\absFeature'^{\mathbb{(A)}}_{S_0} \cup \absFeature'^{\mathbb{(A)}}_{S}$ is sufficient.
\end{enumerate}
\end{proof}
\begin{lemma}
$\forall \suffFeature \subseteq \absFeature$, $\dlower(\suffFeature) \leq \sum\limits_{\absFeature_{n_i} \in \absFeature \setminus \suffFeature} P_{ub}(\absFeature_{n_i})$ where $P_{ub}(\absFeature_{n_i})$ is defined in ~\eqref{eq:apxpriority}.
\label{lem:priorbound}
\end{lemma}
\begin{proof}
\begin{align*}
    \dlower(\suffFeature) &= |\lb(\abstraction) - \lb(\abstraction(W_{l}, S))| \\
    &= \max_{X \in \abstraction_{l-1}} |\sum\limits_{\absFeature_{n_i} \in \absFeature \setminus \suffFeature} C^T W[:i]X| \\
    &\leq \max_{X \in \abstraction_{l-1}} \sum\limits_{\absFeature_{n_i} \in \absFeature \setminus \suffFeature} |C^T W[:i]X| \\
    &\leq \sum\limits_{\absFeature_{n_i} \in \absFeature \setminus \suffFeature} \max_{X \in \abstraction_{l-1}} |C^T W[:i]X| \\
    &= \sum\limits_{\absFeature_{n_i} \in \absFeature \setminus \suffFeature} P_{ub}(\absFeature_{n_i}) \quad [\text{ From } ~\eqref{eq:apxpriority}]
\end{align*}
\end{proof}
\begin{lemma}
A feature set $\suffFeature \subseteq \absFeature$ with $\dlower(\suffFeature) \leq \lb(\abstraction)$ is sufficient provided $\lb(\abstraction) \geq 0$.
    \label{lem:suffbound}
\end{lemma}
\begin{proof}
$\dlower(\suffFeature) = |\lb(\abstraction) - \lb(\abstraction(W_{l}, S))|$. So, there can be two cases:
\begin{enumerate}
    \item $\lb(\abstraction(W_{l}, S)) = \lb(\abstraction) + \dlower(\suffFeature)$. Since, $\lb(\abstraction) \geq 0$ and $\dlower(\suffFeature) \geq 0$, $\lb(\abstraction(W_{l}, S)) \geq 0$. So, $\suffFeature$ is sufficient.
    \item $\lb(\abstraction(W_{l}, S)) = \lb(\abstraction) - \dlower(\suffFeature)$\\
    $\lb(\abstraction) \geq 0$ and $\dlower(\suffFeature) \leq \lb(\abstraction)$. \\
    So, $\lb(\abstraction(W_{l}, S)) \geq 0$. So, $\suffFeature$ is sufficient.
\end{enumerate}
\end{proof}
\begin{lemma}
Let, $P_{max}$ denote the maximum of all priorities $P_{ub}(\absFeature_{n_i})$ over $\absFeature$. \\
Let $\suffFeature \subseteq \absFeature$. If $|\suffFeature| \leq \lfloor \frac{\lb(\abstraction)}{P_{max}} \rfloor$, then proof feature set $\suffFeature^{c} = \absFeature \setminus \suffFeature$ is sufficient provided $\lb(\abstraction) \geq 0$.
    \label{lem:minNonzero}
\end{lemma}
\begin{proof}
\begin{align*}
    &\;\forall \absFeature_{n_i} \in \absFeature, P_{ub}(\absFeature_{n_i}) \leq P_{max} \\
    \text{From Lemma~\ref{lem:priorbound}, }&\; \dlower(\suffFeature^c) \leq |\suffFeature|\times P_{max} \\
    \text{Also, }&\; |\suffFeature| \leq \lfloor \frac{\lb(\abstraction)}{P_{max}} \rfloor \\
    \text{So, }&\;\dlower(\suffFeature^c) \leq \lb(\abstraction) \\
    \text{From Lemma~\ref{lem:suffbound}, }&\; \suffFeature^c \text{ is sufficient.}
\end{align*}
\end{proof}
\textbf{Theorem}~\ref{thm:efficacyBound}:
Given any network $N$ is verified on $(\inreg, \outprop)$ with verifier $\ver$ then $|\cannFeature^{(\alg)}|\leq d_{l-1} - |Z(\absFeature)| - \lfloor \frac{\lb(\abstraction)}{P_{max}} \rfloor $
\begin{proof}
    The algorithm~\ref{alg:featureComputation} arranges the elements of the proof feature set $\absFeature$ in decreasing order according to the priority defined by $P_{ub}$.\\
    Let $\absFeature'$ be the ordered set corresponding to $\absFeature$. So, $\absFeature' = \absFeature_{n_1} :: \cdots :: \absFeature_{n_m}$, where $::$ is the list concatenation.\\
    The elements of $Z(\absFeature)$ will be at the end of this ordering. So, $\absFeature'$ can be written as $\absFeature'' :: Z(\absFeature)$ where $Z(\absFeature) = \absFeature_{n_{k+1}} :: \cdots :: \absFeature_{n_m}$ and $\absFeature'' = \absFeature_{n_1} :: \cdots :: \absFeature_{n_k}$ and $p$ be some of the last elements of $\absFeature''$ s.t. the sum of their priorities just less than $\lfloor \frac{\lb(\abstraction)}{P_{max}} \rfloor$, i.e., \\
    \begin{align*}
        p = \absFeature_{n_j} :: \cdots :: \absFeature_{n_k} \\
        \sum\limits_{i=j}^{k}P_{ub}(\absFeature_{n_i}) \leq \lfloor \frac{\lb(\abstraction)}{P_{max}} \rfloor \\
        \sum\limits_{i=j-1}^{k}P_{ub}(\absFeature_{n_i}) \geq \lfloor \frac{\lb(\abstraction)}{P_{max}} \rfloor
    \end{align*}
    Further, let $p' = p::Z(\absFeature)$, i.e., $p' = \absFeature_{n_j} :: \cdots :: \absFeature_{n_m}$. Since $P_{ub}$ is 0 for all elements of $Z(\absFeature)$, 
    \begin{align}
        \sum\limits_{i=j}^{m}P_{ub}(\absFeature_{n_i}) \leq \lfloor \frac{\lb(\abstraction)}{P_{max}} \rfloor
    \end{align}
    Also, $|p'| = |Z(\absFeature)| + \lfloor \frac{\lb(\abstraction)}{P_{max}} \rfloor$
    Now, we prove by induction on the number of steps of the while loop in the algorithm~\ref{alg:featureComputation} that the set $\cannFeature^{(\alg)}$ never contains any elements from $p'$. \\
    \textbf{Induction Hypothesis: }$\cannFeature^{(\alg)} \cap p' = \{\}$ \\
    \textbf{Base Case: }At initialization, $\cannFeature^{(\alg)} = \{\}$. So, the induction hypothesis holds trivially. \\
    \textbf{Induction Step: }Let the induction hypothesis be true for the $n$-th step of the algorithm~\ref{alg:featureComputation}. For the $(n+1)$-th step, let the new $\absFeature^{\mathbb{(A)}}_{S_0}$ and $\absFeature^{\mathbb{(A)}}_{S}$ be $\absFeature'^{\mathbb{(A)}}_{S_0}$ and $\absFeature'^{\mathbb{(A)}}_{S}$ respectively. Consider the following two cases:
    \begin{enumerate}
        \item Let $\absFeature^{\mathbb{(A)}}_{S_0} \cup \absFeature^{\mathbb{(A)}}_{S_1}$ be sufficient at line 12. In this case, $\absFeature'^{\mathbb{{A}}_{S_0}} = \absFeature^{\mathbb{{A}}_{S_0}}$. So, the induction hypothesis holds.
        \item Let $\absFeature^{\mathbb{(A)}}_{S_0} \cup \absFeature^{\mathbb{(A)}}_{S_1}$ be not sufficient at line 12. \\
        \textbf{Claim: }$\absFeature^{\mathbb{(A)}}_{S_0} \cap p' = \{\}$ \\
        Let the above claim be false. \\
        $\implies \absFeature^{\mathbb{(A)}}_{S_0} \cap p' \not = \{\}$ \\
        $\implies \absFeature \setminus (\absFeature^{\mathbb{(A)}}_{S_0} \cup \absFeature^{\mathbb{(A)}}_{S_1}) \subset p'$ \\
        $\implies \sum\limits_{\absFeature{n_i} \in \absFeature \setminus (\absFeature^{\mathbb{(A)}}_{S_0} \cup \absFeature^{\mathbb{(A)}}_{S_1})} P_{ub} < \lfloor \frac{\lb(\abstraction)}{P_{max}} \rfloor \quad $[From (6)] \\
        $\implies (\absFeature^{\mathbb{(A)}}_{S_0} \cup \absFeature^{\mathbb{(A)}}_{S_1})$ is sufficient. (From Lemma~\ref{lem:minNonzero}) \\
        $\implies$ Contradiction.\\
        So, $\absFeature^{\mathbb{(A)}}_{S_1} \cap p' = \{\}$. In this step, $\absFeature'^{\mathbb{(A)}}_{S_0} = \absFeature^{\mathbb{(A)}}_{S_0} \cup \absFeature^{\mathbb{(A)}}_{S_1}$. Also, from induction hypothesis, $\absFeature^{\mathbb{(A)}}_{S_0} \cap p' = \{\}$.
        Therefore, the induction hypothesis holds, i.e., $\absFeature'^{\mathbb{(A)}}_{S_0} \cap p' = \{\}$.
    \end{enumerate}
\end{proof}

%% file: appendixb.tex
\section{Additional Experiments}
\label{appendix:results}
\subsection{Plots for distributions of the size of proof feature set}
\begin{table}[htbp]
    \centering
    \captionsetup{justification=centering}
    \caption{Distribution of the size of proof feature set for Standardly trained networks on MNIST dataset}
    \resizebox{0.45\textwidth}{!}{
    \begin{tabular}{l | c}
    \toprule
    \Large $\pmb{\epsilon}$ & MNIST \\
    \midrule
    \Large \textbf{0.02} & 
    \includegraphics[]{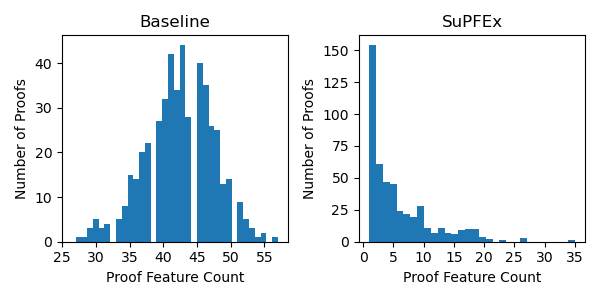}
    \label{fig:mnist-standard-histo1}
    \\
    \midrule
   \Large \textbf{0.1} &
    \includegraphics[]{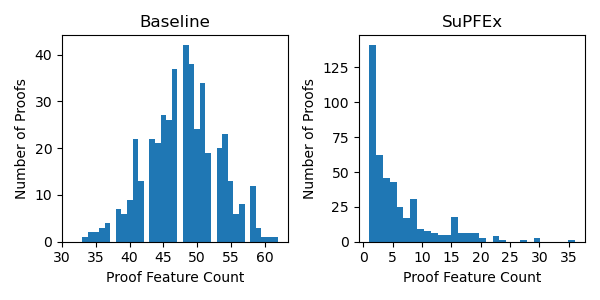}
    \label{fig:mnist-standard-histo2}
    \\
    \midrule
    \end{tabular}}
    \label{tab:standard_mnist}
\end{table}

\begin{table}[htbp]
    \centering
    \captionsetup{justification=centering}
    \caption{Distribution of the size of proof feature set for Standardly trained networks on CIFAR-10 dataset}
    \resizebox{0.45\textwidth}{!}{
    \begin{tabular}{l | c}
    \toprule
    \Large $\pmb{\epsilon}$ & CIFAR-10 \\
    \midrule
    \Large $\mathbf{\frac{0.2}{255}}$ & 
    \includegraphics[]{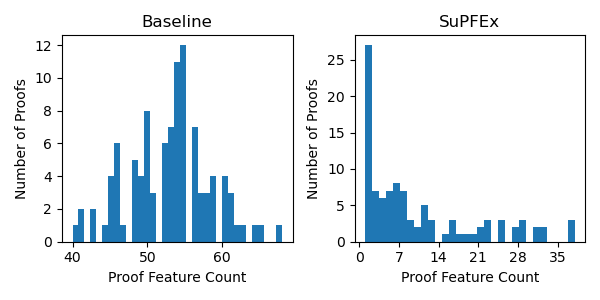}
    \label{fig:cifar-standard-histo1}
    \\
    \midrule 
    \Large $\mathbf{\frac{0.2}{255}}$ &
    \includegraphics[]{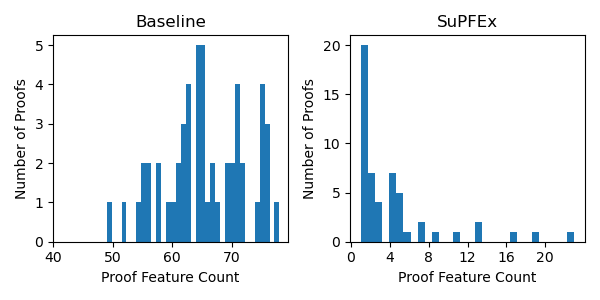}
    \label{fig:cifar-standard-histo2}
    \\
    \midrule
    \end{tabular}}
    \label{tab:standard_cifar10}
\end{table}

\begin{table*}[htbp]
    \centering
    \captionsetup{justification=centering}
    \caption{Distribution of the size of proof feature set for MNIST dataset}
    \resizebox{0.95\textwidth}{!}{
    \begin{tabular}{l  c  c}
    \toprule
    \Large \textbf{Training Method} & \Large \textbf{$\mathbf{\pmb{\epsilon}=0.02}$} & \Large \textbf{$\mathbf{\pmb{\epsilon}=0.1}$} \\
    \midrule
    \Large \textbf{PGD} & 
    \includegraphics[]{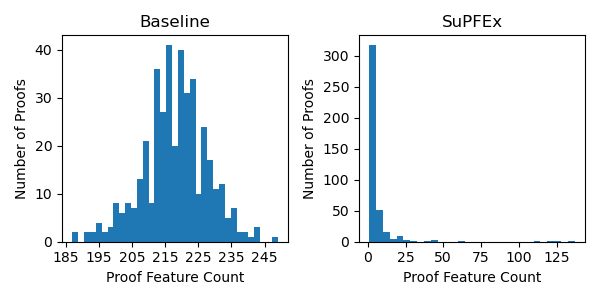}
    \label{fig:mnist-pgd-histo1}
    & 
    \includegraphics[]{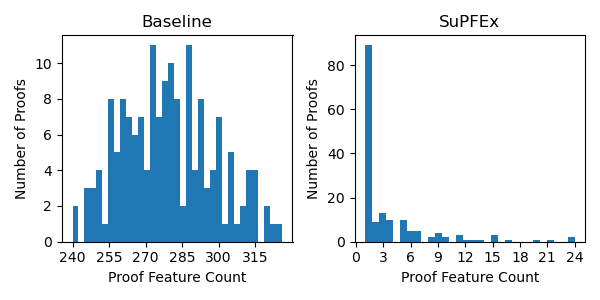}
    \label{fig:mnist-pgd-histo2}
    \\
    \midrule
    \Large \textbf{COLT} & 
    \includegraphics[]{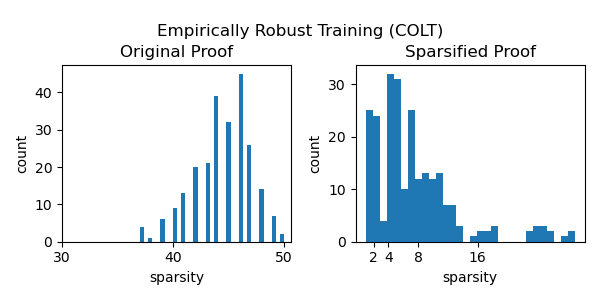}
    \label{fig:mnist-colt-histo1}
    & 
    \includegraphics[]{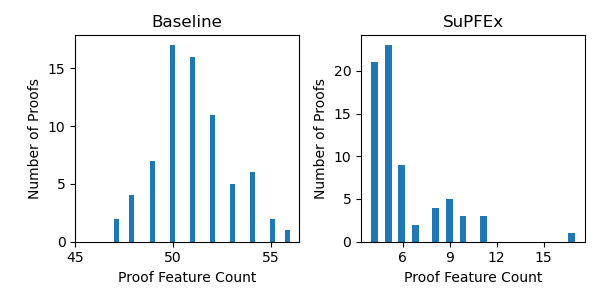}
    \label{fig:mnist-colt-histo2}
    \\
    \midrule
    \Large \textbf{CROWN-IBP} & 
    \includegraphics[]{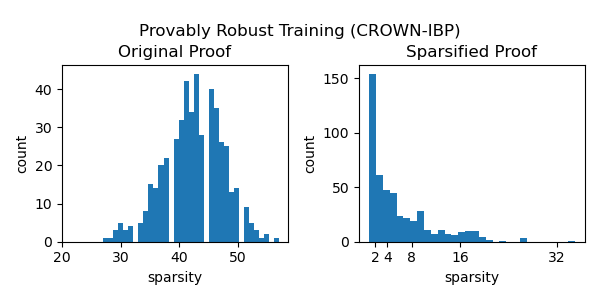}
    \label{fig:mnist-crown-histo1}
    & 
    \includegraphics[]{histPlots/mnist_cnn_2layer_width_1_best.pth_Domain.DEEPZ_0.1.png}
    \label{fig:mnist-crown-histo2}
    \\
    \midrule
    \end{tabular}}
\end{table*}

\begin{table*}[htbp]
    \centering
    \captionsetup{justification=centering}
    \caption{Distribution of the size of proof feature set for CIFAR-10 dataset}
    \resizebox{0.95\textwidth}{!}{
    \begin{tabular}{l c c}
    \toprule
    \Large \textbf{Training Method} & \Large \textbf{$\mathbf{\pmb{\epsilon}=0.2/255}$} & \Large \textbf{$\mathbf{\pmb{\epsilon}=2/255}$} \\
    \midrule
    \Large \textbf{PGD} & 
    \includegraphics[]{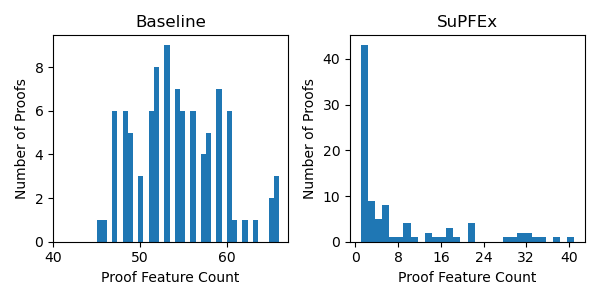}
    \label{fig:cifar-pgd-histo}
    & 
    \includegraphics[]{histPlots/convSmall_pgd_cifar.onnx_Domain.DEEPZ_0.00784313725490196.png}
    \label{fig:cifar-pgd-histo2}
    \\
    \midrule
    \Large \textbf{COLT} & 
    \includegraphics[]{histPlots/cifar10_8_255_colt.onnx_Domain.DEEPZ_0.0007843137254901962.png}
    \label{fig:cifar-colt-histo1}
    & 
    \includegraphics[]{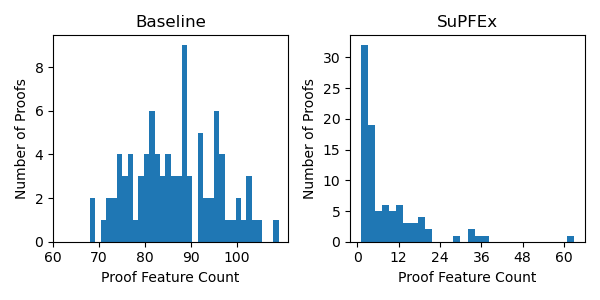}
    \label{fig:cifar-colt-histo2}
    \\
    \midrule
    \Large \textbf{CROWN-IBP} & 
    \includegraphics[]{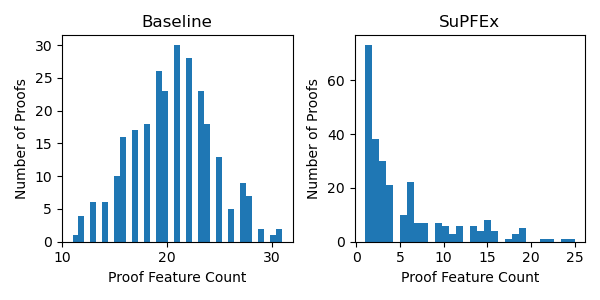}
    \label{fig:cifar-crown-histo1}
    & 
    \includegraphics[]{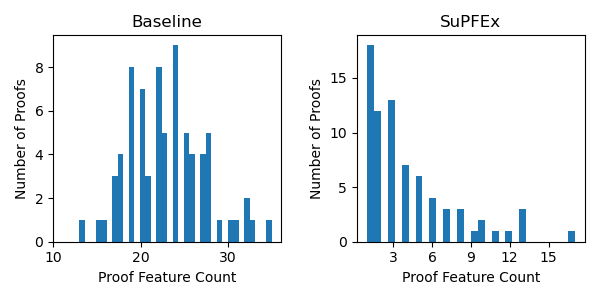}
    \label{fig:cifar-crown-histo2}
    \\
    \midrule
    \end{tabular}}
\end{table*}


\newpage
\subsection{Additional plots for the top proof feature visualization}
\label{sec:addiproofFeatureVis}
\begin{figure}[h!]
    \centering
    \begin{subfigure}[b]{0.5\textwidth}
         \centering
         \includegraphics[width=\textwidth]{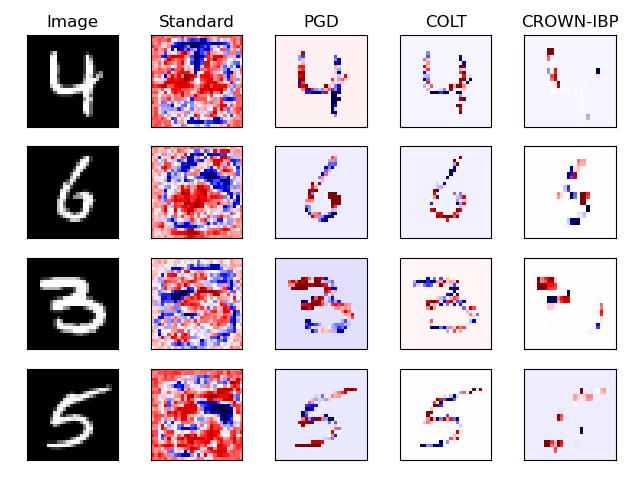}
         \caption{Gradient maps generated on MNIST networks}
         \label{fig:additional_mnist_plots}
     \end{subfigure}
     \hfill
     \begin{subfigure}[b]{0.5\textwidth}
         \centering
         \includegraphics[width=\textwidth]{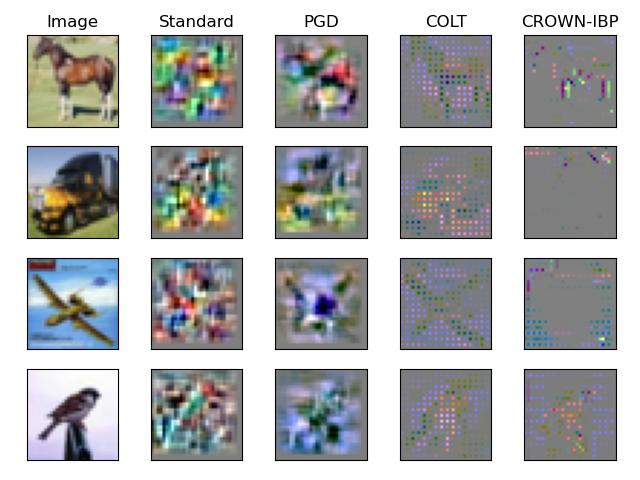}
         \caption{Gradient maps generated on CIFAR-10 networks}
         \label{fig:fig:additional_cifar_plots}
     \end{subfigure}
\caption{Additional plots for the top proof feature visualization (in addition to Fig.~\ref{fig:grad_maps}) - Visualization of gradient map of top proof feature (having highest priority) generated for networks trained with different training methods. It is evident that the top proof feature corresponding to the standard network highlights both relevant and spurious input features. In contrast, the top proof feature of the provably robust network does filter out the spurious input features, but it comes at the expanse of some important input features.
The top proof features of the networks trained with PGD filter out more spurious features as compared to standard networks.
Finally, the top proof features of the networks trained with COLT filter out the spurious input features and also correctly highlight the relevant input features.}
\label{fig:additional_heat_maps}
\end{figure}

\newpage
\subsection{Visualization of the top proof feature for higher $\epsilon$ values}
\label{sec:addiHighEpsPlots}
\begin{figure}[h!]
    \centering
    \begin{subfigure}[b]{0.5\textwidth}
         \centering
         \includegraphics[width=\textwidth]{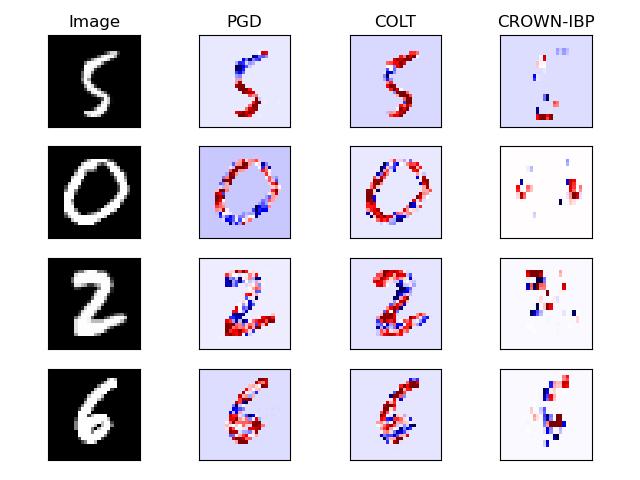}
         \caption{Gradient maps generated on MNIST networks}
         \label{fig:mnist_grad_map_high_eps}
     \end{subfigure}
     \hfill
     \begin{subfigure}[b]{0.5\textwidth}
         \centering
         \includegraphics[width=\textwidth]{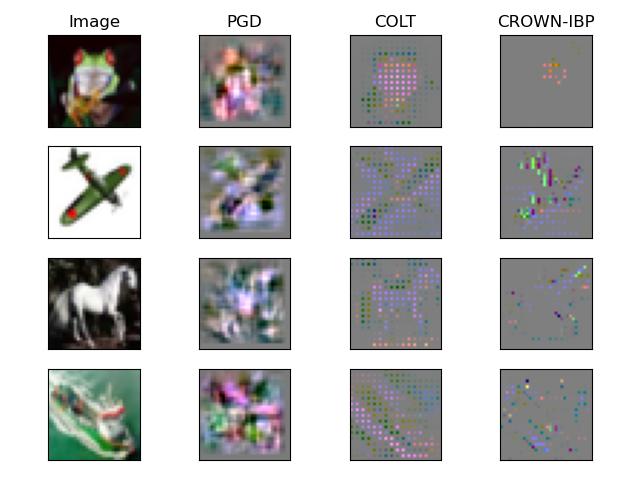}
         \caption{Gradient maps generated on CIFAR-10 networks}
         \label{fig:cifar_grad_map_high_eps}
     \end{subfigure}
\caption{Visualization of gradient map of top proof feature (having highest priority) generated for networks trained with different robust training methods. For these networks, we define local properties with higher $\epsilon$ values. For MNIST networks and CIFAR-10 networks, we take $\epsilon = 0.1$ and $\epsilon = 2 / 255$ respectively.}
\label{fig:grad_maps_with_high_eps}
\end{figure}

\newpage
\subsection{Visualization of multiple proof features from the extracted proof feature set}
\label{sec:addiMultiFeatureVis}
\begin{figure}[h!]
    \centering
    \begin{subfigure}[b]{0.5\textwidth}
         \centering
         \includegraphics[width=\textwidth]{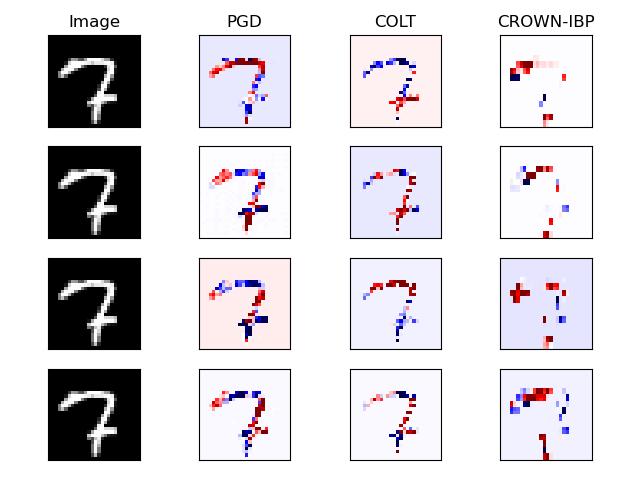}
         \caption{Gradient maps generated on MNIST networks}
         \label{fig:mnist_grad_maps_for_feature_set}
     \end{subfigure}
     \hfill
     \begin{subfigure}[b]{0.5\textwidth}
         \centering
         \includegraphics[width=\textwidth]{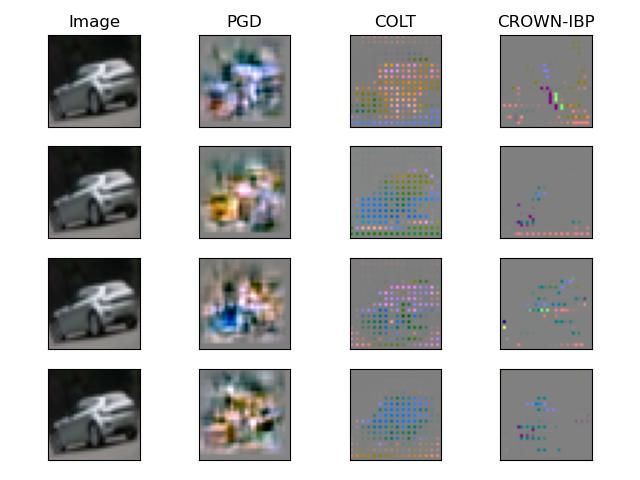}
         \caption{Gradient maps generated on CIFAR-10 networks}
         \label{fig:cifar_grad_maps_for_feature_set}
     \end{subfigure}
\caption{Visualization of gradient maps of top-4 proof features (having highest priority) extracted for networks trained with different robust training methods. The gradient maps of the proof features are presented in decreasing order of priority with the top row showing the gradient map corresponding to the top proof feature of each network. }
\label{fig:grad_maps_for_feature_set}
\end{figure}

\begin{figure}[h!]
    \centering
    \begin{subfigure}[b]{0.5\textwidth}
         \centering
         \includegraphics[width=\textwidth]{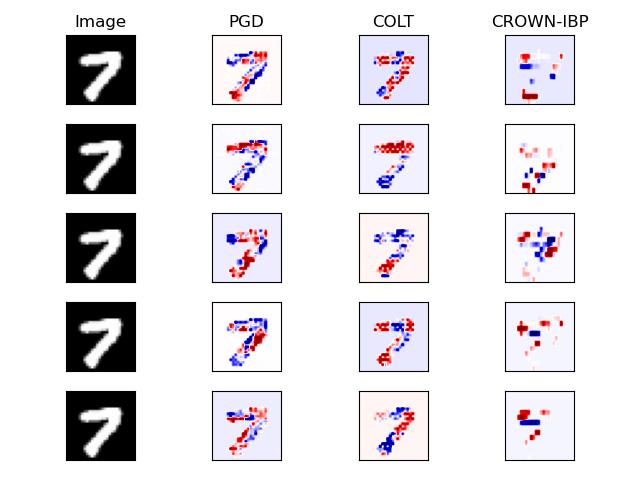}
         \caption{Gradient maps generated on MNIST networks}
         \label{fig:mnist_grad_maps_for_feature_set_additional}
     \end{subfigure}
     \hfill
     \begin{subfigure}[b]{0.5\textwidth}
         \centering
         \includegraphics[width=\textwidth]{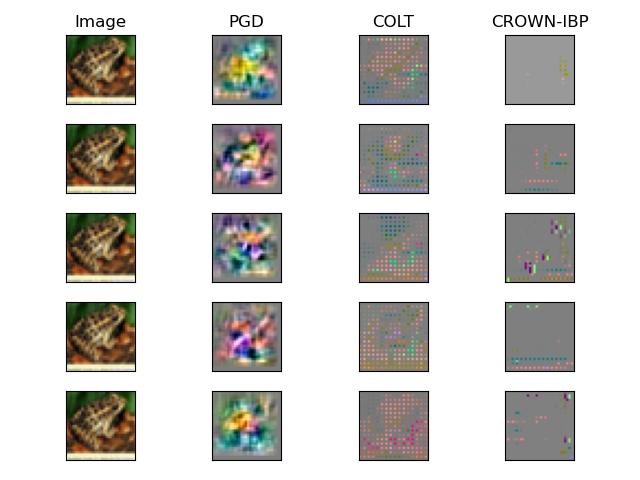}
         \caption{Gradient maps generated on CIFAR-10 networks}
         \label{fig:cifar_grad_maps_for_feature_set_additional}
     \end{subfigure}
\caption{Visualization of gradient maps of top-5 proof features (having highest priority) extracted for networks trained with different robust training methods. The gradient maps of the proof features are presented in decreasing order of priority with the top row showing the gradient map corresponding to the top proof feature of each network.}
\end{figure}

\clearpage
\newpage
\subsection{Comparing proof features with high priority to pruned proof features with low priority}
\label{sec:topBottomFeature}
\begin{figure}[htbp]
\centering
\includegraphics[width=0.5\linewidth, height=0.5\linewidth]{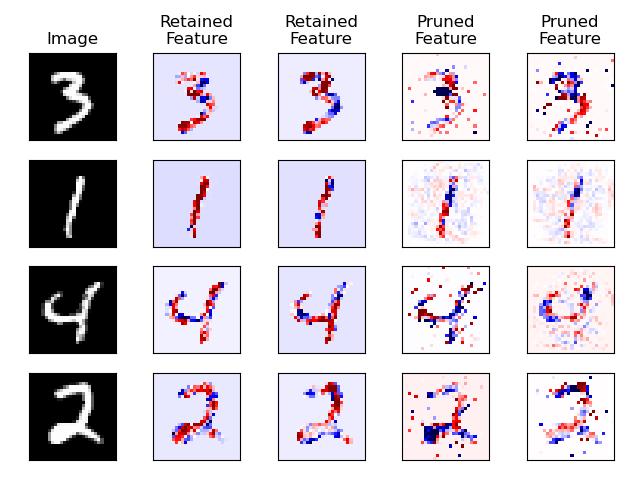}
\caption{Comparing gradients of the top proof features retained by \algname to proof features with low priority. As expected, proof features with low priority contain spurious input features. The shown gradients are computed on COLT-trained MNIST networks.}
\label{fig:mnistTopBottom}
\end{figure}

\clearpage
\newpage
\subsection{Additional plots for sensitivity analysis w.r.t $\pmb{\epsilon_{train}}$}
\label{sec:addSensitivity}
\begin{figure}[htbp]
\centering
\includegraphics[width=0.5\linewidth, height=0.5\linewidth]{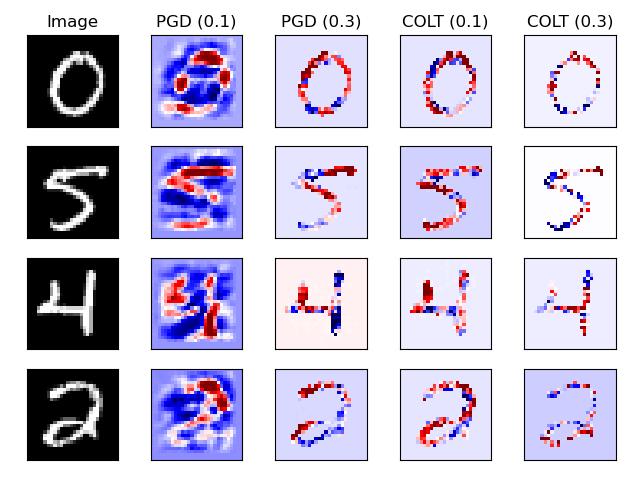}
\caption{Additional plots for visualizing gradients of the top proof feature for PGD and COLT networks trained using different values of $\epsilon_{train} \in \{0.1, 0.3\}$
The gradient map corresponding to the networks trained with the higher value of $\epsilon_{train}$ filter out more input features than the ones trained with smaller $\epsilon_{train}$ value.}
\label{fig:grad_maps_with_changed_eps_additional}
\end{figure}

\clearpage
\newpage
\subsection{Comparing proofs of different verifiers}
\label{sec:additionalVerRes}
We note that for high values of $\epsilon$ i.e. $\epsilon = 0.1$ for MNIST and $\epsilon = 2/255$ for CIFAR-10 most of the properties don't get verified even for robustly trained networks with IBP. Hence, we omitted them for this analysis.

\begin{table}[htbp]
    \centering
    \captionsetup{justification=centering}
    \caption{Comparing proofs of IBP \& DeepZ}
    \resizebox{0.98\linewidth}{!}{
    \begin{tabular}{l l l l l l l l}
    \toprule
    Dataset & Training & Input & \% properties  & \% properties & \% proofs with the & \% proofs with the  & \% proofs with the \\
    \text{} & Method & Region $(\inreg)$ & proved by IBP & proved by DeepZ & same top feature & same top-5 feature & same feature set \\
    \text{} & \text{} & eps ($\epsilon$) & \text{} & \text{} & \text{}& \text{} & \text{} \\
    \midrule
    MNIST & PGD Trained & 0.02 & 26.0 \% & 82.0 \% & 100 \% & 92.0 \% & 92.0 \% \\
    \text{} & COLT & 0.02 & 49.8 \% &  89.4 \% & 100.0 \% & 87.17 \% & 82.05 \% \\
    \text{} & CROWN-IBP & 0.02 &  93.4 \% & 96.4 \% & 99.79 \% & 96.26 \% & 93.15 \% \\
    \midrule
    CIFAR-10 & PGD Trained & 0.2/255 & 10.2 \% & 47.0 \% & 100 \% & 98.31 \% & 96.87 \% \\
    \text{} & COLT & 0.2/255 & 17.2 \% & 53.0 \% & 97.87 \% & 92.53 \% & 89.36 \% \\
    \text{} & CROWN-IBP & 0.2/255 &  21.8 \% & 53.0 \% &  100 \% & 97.92 \% & 95.89 \% \\
     \midrule     
\end{tabular}}
    \label{tab:total_verifier_comparision_table}
\end{table}